\documentclass[letterpaper, 10 pt, conference]{ieeeconf}
\usepackage{cite}
\usepackage{amsmath,amssymb,amsfonts}
\usepackage{algorithm}
\usepackage{algpseudocode}
\usepackage{graphicx}
\usepackage{subcaption}
\usepackage{textcomp}
\usepackage{enumerate}
\usepackage{xcolor}
\usepackage{tikz}
\usepackage{comment}

\let\labelindent\relax
\usepackage[shortlabels]{enumitem}

\IEEEoverridecommandlockouts 
\newcommand{\B}{\mathbb{B}}
\newcommand{\R}{\mathbb{R}}
\newcommand{\N}{\mathbb{N}}
\newcommand{\uu}{\mathcal{U}}

\newcommand{\xx}{\mathcal{X}}
\newcommand{\TT}{\mathcal{T}}
\newcommand{\bb}{\mathcal{B}}
\newcommand{\ii}{\mathcal{I}}
\newcommand{\rr}{\mathcal{R}}
\newcommand{\vv}{\mathcal{V}}
\newcommand{\zz}{\mathcal{Z}}
\newcommand{\hh}{\mathcal{H}}
\newcommand{\G}{\mathcal{G}}
\newcommand{\I}{\mathcal{I}}
\newcommand{\dt}{\delta t}
\newcommand{\dox}{\dot{\overline{x}}}
\newcommand{\on}{\overline{n}}

\newcommand{\dxm}{\dot{x}_m}
\newcommand{\ox}{\overline{x}}

\newcommand{\fm}{f_m}
\newcommand{\Gm}{G_m}
\newcommand{\tx}{\tilde{x}}
\newcommand{\ty}{\tilde{y}}
\newcommand{\tf}{\tilde{f}}
\newcommand{\tG}{\tilde{G}}
\newcommand{\tn}{\tilde{n}}
\newcommand{\of}{\overline{f}}

\newcommand{\xk}{\mathfrak{x}}
\newcommand{\dpsi}{\dot{\psi}}
\newcommand{\dx}{\dot{x}}
\newcommand{\dy}{\dot{y}}
\newcommand{\dv}{\dot{v}}

\newcommand{\ub}{\bold{u}}

\newcommand{\nn}{\mathcal{N}}
\newcommand{\trans}{\mathsf{T}}
\newcommand{\ee}{\mathcal{E}}

\newcommand{\lmax}{L_{0}}

\newtheorem{ass}{\bf Assumption}
\newtheorem{prb}{\bf Problem}
\newtheorem{definition}{\bf Definition}
\newtheorem{thm}{\bf Theorem}
\newtheorem{lem}{\bf Lemma}

\newtheorem{con}{\bf Condition}

\newtheorem{obj}{\bf Objective}
\newtheorem{rem}{\bf Remark}

\title{\LARGE \bf
Reachable Predictive Control: A Novel Control Algorithm for Nonlinear Systems with Unknown Dynamics and its Practical Applications*
}

    \author{Taha Shafa, Yiming Meng, and Melkior Ornik
\thanks{*This work was supported in part by NASA grant 80NSSC22M0070 and the Air Force Office of Scientific Research under award number FA9550-23-1-0131.}
\thanks{Taha Shafa, Yiming Meng, and Melkior Ornik are with the Department of Aerospace Engineering and the Coordinated Science Laboratory, University of Illinois Urbana-Champaign, Urbana, USA. 
        {\tt\small tahaas2, ymmeng, mornik@illinois.edu}}}

\begin{document}

\maketitle

\begin{abstract}

This paper proposes an algorithm capable of driving a system to follow a piecewise linear trajectory without prior knowledge of the system dynamics. Motivated by a critical failure scenario in which a system can experience an abrupt change in its dynamics, we demonstrate that it is possible to follow a set of waypoints comprised of states analytically proven to be reachable despite not knowing the system dynamics. The proposed algorithm first applies small perturbations to locally learn the system dynamics around the current state, then computes the set of states that are provably reachable using the locally learned dynamics and their corresponding maximum growth-rate bounds, and finally synthesizes a control action that navigates the system to a guaranteed reachable state.


\end{abstract}

\section{Introduction}


Many real-world systems must operate in environments where their exact dynamics are unknown, yet reliable motion toward safety-critical or mission-critical goals must still be achieved. In such settings, even modest uncertainty in the dynamics can render conventional model-based planning or tracking strategies ineffective. In this paper, we aim to simulate scenarios where the existing system model is ineffective due to an abrupt changes in its dynamics. This could be due to physical failures or environmental disturbances. 

We address this challenge by proposing \textit{Reachable Predictive Control} (RPC), a resilient framework for safe navigation. RPC proceeds iteratively: (1) update local dynamics through active learning, (2) compute a Guaranteed Reachable Set (GRS)  with limited information and assign a feasible waypoint, (3) synthesize a controller to reach that waypoint, and (4) repeat to steer the system toward a provably reachable region. 
Unlike classical robust safety control, which relies on nominal models perturbed by mismatches or disturbances, RPC assumes only a local Lipschitz bound on the dynamics. This encompasses a broader class of systems and enables rapid response to abrupt 
environmental disturbances.

We review some crucial results 
that are pertinent to the work presented in this paper. In the realm of reachability-based robust safety control, frameworks developed around Control Barrier Functions (CBFs) and  Model Predictive Control (MPC) represent  the most widely adopted approaches. 

\textit{CBFs}  are widely used to enforce safety in dynamical systems, often as filters that minimally adjust control inputs to prevent constraint violations. While elegant in theory, constructing valid CBFs is difficult, though integration with data-driven controllers is attractive. Recent work has explored neural-network-based barrier learning through policy evaluation and Hamilton–Jacobi reachability \cite{lin2024one, ramdani2009computing, rungger2018accurate}, extending to parametric uncertainties and dynamic environments. These methods, however, face challenges in dynamics assumptions, verification complexity, and interpretability. The approach in \cite{knoedler2025safety} eases reliance on precise dynamics and proposes runtime CBF approximations, but without provable safety guarantees.

\textit{MPC}   enforces safety by embedding state and input constraints into a finite-horizon problem solved in receding-horizon form \cite{rawlings2020model, mayne2000constrained}. Robust variants such as tube-based \cite{mayne2005robust} and min-max MPC \cite{bemporad1999control} handle uncertainty when a nominal model is available. Extensions combine MPC with CBFs \cite{quan2021robust} or reachability-based safety certification \cite{rosolia2018data}. Recent work leverages meta-learning for system identification and Koopman-based models to adapt to nonlinear systems with parametric uncertainties \cite{wang2023learning, dittmer2022data, chiu2022collision}. Despite these advances, MPC still suffers from finite-horizon approximations, the need for long horizons, and heavy computational demands in nonlinear or high-dimensional settings.

\textit{Learning-based Safety Control}, particularly reinforcement learning (RL), has shown success in generating robust control policies for high-dimensional dynamics, sensory feedback, and complex tasks \cite{cheng2019end, zhao2021model, thananjeyan2021recovery}, including agility, robustness, and terrain traversal \cite{he2024agile, bharadhwaj2020conservative, hsu2023sim}. However, many RL approaches emphasize performance over safety—whether through motion planner integration or collision penalties \cite{jenelten2024dtc, miki2022learning, agarwal2023legged}, and thus inherit model-based drawbacks, restrict mobility, and lack formal safety guarantees. the resulting policies often exhibit degraded safety when deployed in the real world, especially under distribution shifts away from the training environment.

Unlike prior work, the proposed RPC framework uses reachability analysis with provable guarantees, requiring neither a nominal model nor pre-collected data, enabling real-time  response to a broader class of uncertainties. 

\subsection{Contributions}

We formally present the RPC algorithm, an algorithm capable of solving reach-avoid and trajectory-tracking problems without the use of a nominal model. Previous work estimated the GRS using proxy systems \cite{shafa2022reachability, shafa2022maximal, Shafa23Reachability}, later incorporating myopic control \cite{ornik2019control} and developing synthesis methods \cite{meng2024online}, but these pipelines could not handle practical, underactuated systems. We generalize results to incorporate underactuated systems, provide a new algorithm capable of optimal path-following under constraints, and demonstrate its operation in an unknown environment via simulation.

\section{Preliminaries}\label{sec: problem_statement}


In this section, we first introduce the dynamical systems under consideration and detail, through assumptions, what information can be utilized in place of nominal system dynamics. We then state the robust planning problem with the definition of the GRS. As a means of computing the GRS and synthesizing the corresponding control of the reachable path, we summarize the resilient reachability problem and the reachable predictive control framework, which together provide a solution to the stated reachability problem.

\subsection{Unknown Nonlinear System Structure}


We begin by considering a general autonomous, underactuated, nonlinear control-affine dynamical system of the abstract form described as
\begin{small}
    \begin{equation}\label{eq: nonlinaer_control_affine_system}
    \dx(t) = f(x(t)) + G(x(t))u(t),\quad x(0) = x_0,
\end{equation}
\end{small}

\noindent which can be decomposed as 
\vskip -10pt

\begin{small}
    \begin{subequations} \label{eq: underactuated_nonlinear_system}
\begin{align}
&\dot{\tx}(t) = \tf(\tx(t)) + \tG(\tx(t)) u, &
\tx(0) &= \tx_0 \label{eq: underactuated_nonlinear_system_a} \\
&\dot{\ox}(t) = \of(x(t)) = F(x(t))\tx(t), \quad&
\ox(0) &= \ox_0 \label{eq: underactuated_nonlinear_system_b}
\end{align}
\end{subequations}
\end{small}


\noindent such that $x(t) = \begin{bmatrix}\tx(t) & \ox(t)\end{bmatrix}^\trans$ where all $t~\geq~0$, $x(t)~\in~\mathcal{X}\subseteq\R^n$, $\tx(t)~\in~\R^m$, $\ox(t)~\in~\R^{n - m}$.  The admissible control inputs satisfy $u(t)~\in~\uu~=~\B^m(0;1)$. The mappings $f:\R^n~\to~\R^n$ and $G:~\R^n~\to~\R^{n\times m}$ that are differentiable almost everywhere and 
local Lipschitz on a neighborhood $\nn$ of any $x_0 \in \R^n$, with local Lipschitz constants $L_f^\nn$ and $L_G^\nn$. This implies that $\tf$ and $\of$ are locally Lipschitz with constant $L_f^\nn$ as well. In practice, $L_f^\nn$ can be proved to be a valid Lipschitz bound for $F(x(t))$ as well; additional analysis and discussion can be found Lemma \ref{lem: growth_rate_bound} of Appendix \ref{sec: error}. 
A solution to \eqref{eq: nonlinaer_control_affine_system} with control input $u$ and initial value $x_0$ is denoted by $\phi_u(\cdot, x_0)$. Notice that subsystem \eqref{eq: underactuated_nonlinear_system_a} is not dependent on subsystem \eqref{eq: underactuated_nonlinear_system_b}, which is consistent with many practical system models for systems like unmanned drones and ground vehicles \cite{liu2013survey,spong2020robot}. The decoupled structure of \eqref{eq: underactuated_nonlinear_system_a} is required for the results presented in Section~\ref{sec: underapproximated_guarantees}; generalizing the structure to all control-affine systems remains for future work. 
\subsection{Problem Formulation}

We begin by formalizing the strict control design constraints for the nonlinear control of system \eqref{eq: nonlinaer_control_affine_system}.

\begin{con}\label{req: unknown_dynamics}
    There is no known nominal model for system \eqref{eq: nonlinaer_control_affine_system} and its corresponding subsystem \eqref{eq: underactuated_nonlinear_system}. \hfill $\square$
\end{con}

While we may not have access to a viable dynamics model, work in \cite{ornik2019control, meng2024online} demonstrates how piecewise-linear control affine inputs can be used to determine $f(x_0)$ and $G(x_0)$ using trajectory data from a single system run. Since we are limited with no model or a priori trajectory data, our control-oriented learning method is constrained to the following condition.

\begin{con}\label{req: single_system_run}
    There is only one system run, starting from a predetermined initial state $x_0$. All controller synthesis must be performed using trajectory information available during solely that run. \hfill $\square$
\end{con} 

Under the conditions above, our aim is to present a solution to the following problem.

\begin{prb}[Reach-avoid problem]\label{prb: reach_avoid_problem}
    Let $x_0~\in~\xx$ and $T~\geq~0$. Consider the target se $\TT \subseteq \xx$  and the unsafe (blocking) set $\bb\subseteq \xx$, respectively. Find, if it exists, a control signal $u^*:[0,T]\to\uu$ such that the following holds:
    \begin{enumerate}[i)]
        \item $\phi_{u^*}(t,x_0) \notin \bb$ for all $t \in [0,T]$;
        \item there exists $0 \leq T'_u\leq T$ such that $\phi_{u^*}(t,x_0) \in \TT$ for all $t \in [T'_{u^*},T]$;
        \item $T'_u$ from ii) is the minimal such value for all the control laws $u:[0,T] \to \uu$ which satisfy i) and ii). \hfill $\square$
    \end{enumerate}
\end{prb}

A previously developed myopic control algorithm \cite{ornik2019control} aimed to solve Problem \ref{prb: reach_avoid_problem}, however myopic control lacks guarantees. Specifically, it can synthesize control action aimed at reaching an unattainable state, which could cause system \eqref{eq: nonlinaer_control_affine_system} to inadvertently follow a trajectory which intersects $\bb$, thus violating i) in Problem~\ref{prb: reach_avoid_problem}. To avoid this possibility, present the following additional problem and propose a solution in subsequent sections.

\begin{prb}[Trajectory Tracking]\label{prb: partial_trajectory_tracking}
    Let $\hat{\phi}_u(\cdot,x_0):[0,\infty) \to \R^n$ be a piecewise linear trajectory with its time derivatives bounded and let $r > 0$. Design a controller such that $\|\phi_u(t,x_0) - \hat{\phi}_u(t,x_0)\| < r$ for every $t$. \hfill $\square$
\end{prb}

We detail through the following assumption what information we can utilize to solve Problems \ref{prb: reach_avoid_problem} and \ref{prb: partial_trajectory_tracking} under Conditions \ref{req: unknown_dynamics} and \ref{req: single_system_run}. 

\begin{ass}\label{ass: growth_rate_bounds}
    The bounds $L_f^\nn$ and $L_G^\nn$ are known for some compact neighborhood $\nn$ of any $x \in \R^n$ as well as values $f(x_0)$ and $G(x_0)$ such that $G(x_0) \neq 0$.  \hfill $\square$
\end{ass}

This mild assumption is reasonable in practice, as it aligns with the physical limits of the controlled object and captures its inherently  bounded response to changes in state and control inputs. For information on gathering the assumed knowledge, we refer the reader to related literature \cite{ornik2019control, el2023online}. With this knowledge, we want to determine the states we can reach by underapproximating the \textit{guaranteed reachable set}.

\begin{definition}[Guaranteed Reachable Set]
    Let us denote $\mathcal{D}_{\mathrm{con}}$ as the set of all pairs $(f, G)$ such that $\|f(x) - f(x_0)\| \leq L_f^\nn\|x - x_0\|$ and $\|G(x) - G(x_0)\| \leq L_G^\nn\|x - x_0\|$. The Guaranteed Reachable Set (GRS) is the set of all states that can be reached by every dynamic pair $(f,G) \in \mathcal{D}_{\mathrm{con}}$ for system \eqref{eq: nonlinaer_control_affine_system} at time $T$. \hfill $\square$
    
\end{definition}

The rest of the paper introduces an algorithm that solves the trajectory-tracking and reach-avoid problems for a nonlinear control-affine system, requiring knowledge of the dynamics solely through the information provided in Assumption~\ref{ass: growth_rate_bounds}. 
Specifically, to form a key step of the RPC framework and to extract, to the largest extent, the information provided by Assumption~\ref{ass: growth_rate_bounds}, we need to construct a proxy system whose reachable sets are contained within that of the original system \eqref{eq: nonlinaer_control_affine_system}. The next section derives this proxy system and presents novel results that generalize our previous analyses to underactuated systems.

\subsection{Vehicle Dynamics}\label{Sec: vehicle_dynamics}

We aim to demonstrate the capabilities of the proposed approach using an unmanned vehicle and now present the dynamics model employed in the simulation. We emphasize that the vehicle dynamics are solely used to mirror realistic conditions and   physical constraints in simulation, while our proposed algorithm follows a desired trajectory without knowledge of the described model.  

We implement a kinematic bicycle model, frequently used in the state-of-the-art \cite{polack2017kinematic,yuan2014trajectory,matute2019experimental}, written as:

\vskip -4mm

\begin{small}
    \begin{subequations}\label{eq:kinematic_bicycle_model}
    \begin{align}
        \dx &= v\cos({\theta}), \;\dy = v\sin({\theta})\\
        \dot{\theta} & = \frac{v}{l_r}\tan({u_2}), \;\dv = \mathrm{min}(u_1,\mu g) - r_c g
    \end{align}
\end{subequations}
\end{small}

\noindent to simulate vehicle motion. In \eqref{eq:kinematic_bicycle_model}, $(x,y) \in \R^2$ are the position coordinates of the center of mass (CoM), $\theta$ is the yaw angle, $v$ is the linear speed of the vehicle, $l_f$ and $l_r$ are the distance from the CoM to the front and rear axles respectively, $r_c$ is the rolling resistance coefficient, $\mu$ is the coefficient of surface friction, $g$ is the gravitational constant, $u_1$ is the acceleration of the CoM, $u_2$ is the front wheel steering angle. 
Notice that system~\eqref{eq:kinematic_bicycle_model} is not control-affine and thus seemingly violates our assumptions. In practice, this is not an issue, as a locally reduced version of the original system can be found that satisfies the assumed dynamics. Its GRS is contained within that of the original system, providing a conservative but acceptable estimate that keeps the synthesized control paths within a safe region. Further discussion is given in Section~\ref{sec: simulated_results}. 

\section{Proxy System Dynamics}\label{sec: underapproximated_guarantees}

To perform trajectory tracking and obstacle avoidance for an unknown nonlinear system, we first  take the knowledge from Assumption \ref{ass: growth_rate_bounds} to derive proxy systems whose reachable sets are contained in the GRS. Then, following the pipeline described in the introduction, we use the  proxy system to conservatively estimate the reachable range, identify a relatively safe  waypoint within the GRS based on that estimate, and generate reachable paths to be followed.  Controller synthesis designed to follow these paths is presented in the next section, but first, for the convenience of the reader, we revisit the underapproximated proxy control system for the GRS evaluation \cite[Theorem 1]{shafa2022reachability}, which we can use to provide a set of reachable states for subsystem \eqref{eq: underactuated_nonlinear_system_a}.

\begin{thm}\label{thm: ball_underapproximation}
Let $\uu$, $L_f^\nn$, and $L_G^\nn$ be given. Then,   there exists a $\hat{u}\in\uu$ satisfying     \begin{small}
               \begin{equation}
         ku=\tf(\tx) -\tf(\tx_0) + \tG(\tx)\hat{u} 
     \end{equation} 
      \end{small} 
      
      \noindent for all  $\tx\in\B:=\{\tx: \|\tx\|\leq \|\tG^\dagger(\tx_0)\|^{-1}/(L_f^\nn+L_G^\nn)\}$, 
      any $u\in\B^m(0;1)$, and any $k$ such that $|k|\leq \|\tG^\dagger(\tx_0)\|^{-1}-(L_f^\nn+L_G^\nn)\|\tx\|$. \hfill $\square$
\end{thm}

Consequently, by \cite[Theorem 3]{shafa2022reachability}, an underapproximation of the GRS 
of \eqref{eq: underactuated_nonlinear_system_a} can be calculated based on a proxy equation  of the form 
\begin{small}
    \begin{equation}\label{E: proxy}
        \dot{\tx}_p(t) = a + (b - c\|\tx_p(t)\|) u(t), \;\;\tx_p(0)=\tx_0, 
    \end{equation}
\end{small}

\noindent on the domain $\B$, where $a = \tf(\tx_0)$, $b:=\|\tG^\dagger(\tx_0)\|^{-1}$, and  $c: = L_f^\nn+L_G^\nn$. To identify a complete underapproximation of the GRS of system \eqref{eq: nonlinaer_control_affine_system}, we need to extend these results to include the states $\ox(t)$ of \eqref{eq: underactuated_nonlinear_system_b}. We begin by presenting a generalization of \cite[Theorem 1]{shafa2022reachability} to systems with input sets $\uu$ such that $\uu = \B^m(a;b(t))$.

\begin{thm}\label{thm: generalized_gvs_ball_underapproximation}
    Let $\mathcal{U} = \B^m(a;b(t))$ where $a \in \R^m$ such that $b(t)\leq \|a\|$ for all $t$ and $c(t) = \|a\| - b(t)$. Let $L_f^\nn$, and $L_G^\nn$ be defined as above. Let $x \in \mathbb{R}^n$ satisfy $(L_f^\nn + L_G^\nn)\|x\| < \|G^\dagger(x_0)\|^{-1}$. Define \begin{small}
        $\overline{\mathcal{V}}^\mathcal{G}_{x} =   \mathbb{B}^{n}\left(f(x_0); c(t)\left(\|G^\dagger(x_0)\|^{-1} - (L_f^\nn + L_G^\nn)\|x\|\right)\right) \cap \mathrm{Im}(G(x_0))$. 
    \end{small}
Then, $\overline{\mathcal{V}}^\mathcal{G}_{x} \subseteq \mathcal{V}^\mathcal{G}_{x}$. \hfill $\square$
\end{thm}

Leveraging this fact, we next present an underestimate of the reachable set for the actuated part of the system, building upon its connection to another system whose reachable set is a ball contained within the reachable set of \eqref{E: proxy}.

\begin{lem}\label{lem: spherical_reachable_set}
Consider  control systems  $\dot{x}(t) = a + (b - c\|x(t) - x(0)\|)u$ and $\dot{z}(t) = (b - c\|x(t) - x(0)\| - \|a\|)u$, both defined on some ball $\B \subset \R^m$ with $a,x \in \R^m$, $b,c \in\R$, and $u\in\mathcal{U}$.  Let $\rr_x(T,x_0)$ and $\rr_z(T,x_0)$ denote their reachable sets,  
respectively. Suppose the same assumption holds as in Theorem~\ref{thm: generalized_gvs_ball_underapproximation}.  Then, $\rr_z(T,x_0) \subseteq \rr_x(T,x_0).$ Additionally, $\rr_z(T,x_0)$ is a ball in $\B\subset\R^m$. \hfill $\square$
\end{lem}

By immediate virtue of Lemma \ref{lem: spherical_reachable_set}, and using notation consistent with system (5), we identify the proxy system
\begin{equation}\label{E: proxy_sphere}
    \dot{\tx}_p(t) = (\tilde{b} - c\|\tx_p(t)\|) u(t), \;\;\tx_p(0)=\tx_0,
\end{equation}

\noindent where $\tilde{b} = b - \|a\|$, whose reachable set forms a ball contained within the GRS of   \eqref{eq: underactuated_nonlinear_system_a}. Let $\B^m(\tx_0,b(t))$ be the reachable set of \eqref{E: proxy_sphere}. By following similar steps as above, a straightforward application of Theorem \ref{thm: generalized_gvs_ball_underapproximation} using $\B^m(\tx_0,b(t))$ yields the following result.

\begin{thm}\label{thm: unactuated_grs_approximation}
    Consider a control system of the form
\vspace{-2.5mm}
    \begin{equation}\label{eq: ode_unactuated_grs_underapproximation}
        \dot{\ox} = h(\|x\|)v,\quad x(0) = x_0
    \end{equation}

    \noindent on $\{x~|~\|x - x_0\| \leq \|F^\dagger(x_0)\|/L_f^\nn,\,\|\tx_0\| \geq \|\tx_0 - \tx(t)\|\}$, with $v~\in~\B^m(\tx_0,b(t))$ and where $h(\|x\|) = (\|\tx_0\| - \|\tx_0 - \tx(t)\|)(\|F^\dagger(x_0)\| - L_f^\nn\|x - x_0\|)$ if $\|x - x_0\| \leq L_f^\nn/\|x - x_0\|$. The reachable set of   \eqref{eq: ode_unactuated_grs_underapproximation} is contained in the GRS of \eqref{eq: underactuated_nonlinear_system_b}. \hfill $\square$
\end{thm}

Provided these results, we naturally arrive at the following conclusion.

\begin{thm}\label{thm: full_grs_underapproximation}
    Let $\rr^{\G}(T,x_0)$ be the GRS of system \eqref{eq: nonlinaer_control_affine_system}. If $\rr_{\tx}(T,\tx_0)$ and $\rr_{\ox}(T,x_0)$ are the reachable sets of \eqref{E: proxy} and \eqref{eq: ode_unactuated_grs_underapproximation},  and $\overline{\rr}_x(T,x_0) = \rr_{\tx}(T,\tx_0)~\cup~\rr_{\ox}(T,x_0)$, then $\overline{\rr}_x(T,x_0)~\subseteq~\rr^{\G}(T,x_0)$. \hfill $\square$
\end{thm}

Theorem \ref{thm: full_grs_underapproximation} provides a complete set of reachable states for all states in system \eqref{eq: nonlinaer_control_affine_system} using solely the information in Assumption \ref{ass: growth_rate_bounds}. This is achieved by first computing a decomposed GRS estimation for \eqref{eq: underactuated_nonlinear_system_a} under the given input bound, and then using this estimated GRS as the input bound for \eqref{eq: underactuated_nonlinear_system_b}, thereby obtaining the complete set of reachable states. In the next section, we detail how $\rr_{\tx}(T,\tx_0)$ is used in controller synthesis. Then, in our practical demonstration, we will show how $\rr_{\ox}(T,x_0)$ informs high-level decision making, that is, informs which reachable paths to follow.

\section{Controller Synthesis}\label{Sec: algorithm}

We now introduce an algorithm, originally presented in \cite{meng2024online}, that synthesizes control inputs designed to follow piecewise-linear trajectories for system \eqref{eq: underactuated_nonlinear_system_a}. In view of Theorem \ref{thm: ball_underapproximation} and \cite[Theorem 3]{shafa2022reachability}, for any path generated by the proxy system \eqref{E: proxy}, there exists a control signal $u$ for \eqref{eq: underactuated_nonlinear_system_a} that ensures the system follows the same path. 
We detail the design of control inputs for following such reachable paths.

\begin{figure}[t]
  \centering
  \includegraphics[width=0.48\linewidth]{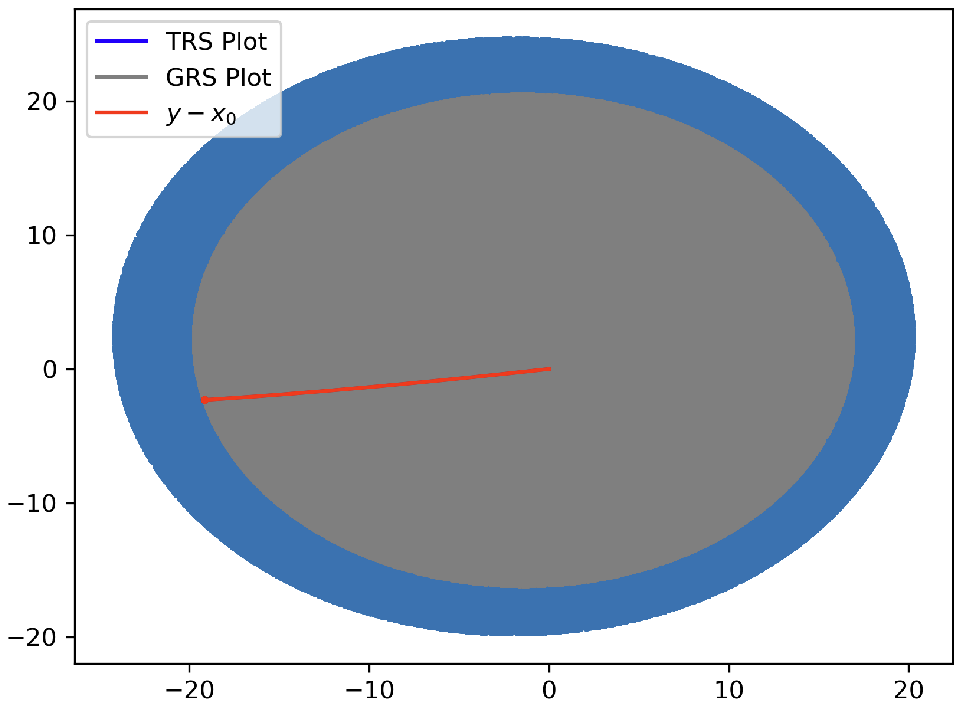}
  \includegraphics[width=0.48\linewidth]{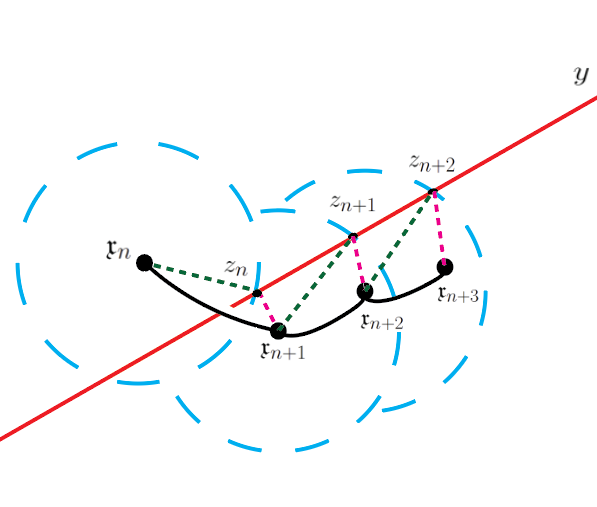}
  \caption{Trajectory following for unknown systems where the line (red) represents the desired reachable linear path contained within the GRS (gray). The  $\xk_{n} = \tx(\tau_n)$ denotes the begining point of each learning cycle, while $z_n$ represents the sequence of points automatically generated by Algorithm \ref{alg: one_time}, each located an $r$-distance apart and progressively approaching $y:=\tx_f$, toward which the system state converges within each learning cycle.}
  \label{fig: grs_trajectory_following}
\end{figure}

\subsection{Control Design}


The results in \cite{meng2024online} can be summarized as follows. For sampling period $\delta t>0$, each learning cycle has length  $\tau=(m+1)\delta t$ with start time  $\tau_n:= n\tau$. Within $[\tau_n+j\delta t, \tau_n+(j+1)\delta t)$ for each $j\in\{1, 2,\cdots, m\}$, piecewise constant controls (sample-and-hold) are applied as 
\vspace{-2.5mm}
\begin{equation}\label{E: u_nj}
u_{n,j}:=u_{n,0} + \Delta u_j,
\end{equation}
where $\Delta u_j:=\pm\epsilon e_j$ generates local excitation along orthonormal unit vectors with a small amplitude 
$\epsilon>0$ for dynamics learning. $\{u_{n,0}\}$
 initializes each cycle and steers the state toward points $\{z_n\}$ on the reference path. Each $z_n$ is designed and dynamically updated to be  placed at an $r$-distance from each cycle’s endpoint and to monotonically approaching  $\tx_f$ along the reference path (see Fig. \ref{fig: grs_trajectory_following} for an illustration), where $r = r(k, \delta t) = \sup_{{u}}|\tx_p(k(m+1)\delta t)|$.

In determining $u_{n,0}$,  we choose \begin{small}
    $u_{0, 0}= (1~-~\epsilon)~ \frac{\tG^\dagger(\tx_0) (\tx_f -\tx_0)}{\|\tG^\dagger(\tx_0)\||\tx_f - \tx_0|}$ 
\end{small} which pushes the state $\tx$ directly toward $\tx_f$. For each $n\geq 1$, parameters 
$k$, $\epsilon$, and $\delta t$ are tuned so that the above described sequence 
above-described $\{z_n\}$ is generated automatically, and $u_{n,0}$ ensures $\dot{d}_{z_n}(\tx(t)) < 0$ for all $t \in [\tau_n,\tau_{n+1})$, where $d_{z_n}(\tx) = \|\tx - z_n\|^2$. The work \cite{meng2024online} established that the mechanism holds under a sufficient condition
\begin{small}
    \begin{equation}\label{eq: convergence_condition}
\begin{gathered}
   \ee_n(\dt, \epsilon)+\ee_\mu(\dt, \epsilon)\leq r(b-c\|\tx(\tau_n)\|-|a|-\ee_r(\dt, \epsilon)+\epsilon M_0),
\end{gathered}
\end{equation}
\end{small} 

\noindent   where   $-r(b-c\|\tx(\tau_n)\|-\|a\|)$ is the guaranteed minimal force of $\dot{d}_{z_n}$ that an optimal control can generate at $\tau_n$, offset by perturbations (with the satisfaction of \eqref{eq: convergence_condition} ensuring   $\dot{d}_{z_n}$ remains negative throughout each learning cycle). The perturbation terms have the following heuristic interpretation (see Appendix~\ref{sec: error} for details):
\begin{itemize}
    \item $\ee_r(\delta t, \epsilon)$ (due to estimation error of r.h.s. of \eqref{eq: underactuated_nonlinear_system_a}),
    \item $\ee_n(\delta t, \epsilon)$ (due to estimation error of $\tx$),
    \item $\epsilon M_0$ (effect of myopic control), and 
    \item $\ee_\mu(\delta t,\epsilon)$ (deviation due to suboptimal control). 
\end{itemize}
Particularly, the term 
$\ee_\mu(\delta t,\epsilon)$ quantifies the performance gap between this suboptimal solution and the true optimal control. Since the learning–control process cannot achieve the exact optimal strategy for minimizing $\dot{d}_{z_n}(\tx(\tau_n))$ (whose minimal value is strictly less than the attraction force $-r(b-c\|\tx(\tau_n)\|-|a|)$), the controller instead seeks a near–optimal solution. Specifically,
$u_{n+1, 0}=(1-\epsilon)\operatorname{argmin}_{\lambda}\langle 2(\tx(\tau_{n+1})-z_{n+1}), \sum_{j=0}^m\lambda_j(\tx_{n,j+1}-\tx_{n,j}) \rangle,$
where  $\lambda$ is such that  $\sum_{j=0}^m\lambda_j=1 \;\text{and}\;\lambda_j\geq 0$ for all $j$, and  
$\tx_{n,j}:=\tx(\tau_n + j\delta t)$ is the shorthand notation of  state   $\tx$. 

Note that all of the perturbation terms introduced above can be tuned to arbitrarily small values, thereby providing a theoretical guarantee for the automation of the control synthesis of each $u_{n,0}$, and hence for the entire control signal in the reachability problem. We summarize the algorithm below and demonstrate its application in the next section, and kindly refer the reader to \cite{meng2024online} for further technical details.

\setcounter{algorithm}{0}
\begin{algorithm}[H]\label{alg: alg}
	\caption{Control Synthesis}\label{alg: one_time} 
	\begin{algorithmic}[1]
      \Require $d$, $m$,  $\tx_{0, 0}:=\tx_0$, $T$, $\tx_f\in\partial \rr_{\tx}(T,\tx_0)$, $\theta_0=0$, and $u_{0, 0}= (1-\epsilon) \frac{\tG^\dagger(
      \tx_0) (\tx_f -\tx_0)}{\|\tG^\dagger(\tx_0)\||\tx_f -\tx_0|}$. 
		\Require $\dt$, $\epsilon$, $k$  based on the condition \eqref{eq: convergence_condition}, where $r = \sup_{u}|\tx_p(k(m+1)\dt)|$

  \State $n=0$.
  \Repeat 
  \State $\tau_n = n(m+1)\dt$.
  \For{$j$ \textbf{from} 0 \textbf{to} $m$}
 \State  $\ub(t)\equiv u_{n,j}$ for all $t\in[\tau_n + j\dt, \tau_n+(j+1)\dt)$ based on Eq.~\eqref{E: u_nj};
  \State $\tx_{n,j+1}=\tx(\tau_n + (j+1)\delta t)$.
  \EndFor
 \State  $\xk_{n+1}=\tx_{n,m+1}$.
 \State Determine $z_{n+1}$ by solving $\|z_{n+1}-\xk_{n+1}\|^2=r^2$ with $z_{n+1}=\theta_{n+1}\tx_f$ and $\theta_{n+1}>\theta_n$. 
 \State Let 
      $u_{n+1, 0}=(1-\epsilon)\operatorname{argmin}_{\lambda}\langle 2(\xk_{n+1}-z_{n+1}), \sum_{j=0}^m\lambda_j(\tx_{n,j+1}-\tx_{n,j}) \rangle,$
 where $\sum_{j=0}^m\lambda_j = 1$ and $\lambda_j\geq 0$ for all $j$.
 \State $n:=n+1$.
 \Until $\|z_n-\tx_f\|<r$. 
	\end{algorithmic}
\end{algorithm}

\section{Simulation of Reachability Predictive Control Framework}\label{sec: simulated_results}

We now demonstrate RPC and its capability to solve the trajectory tracking and reach-avoid problem without knowledge of the system dynamics via simulation. 

We consider a scenario in which an autonomous vehicle carrying a delicate payload must navigate an uncertain environment. To accomplish the mission, the vehicle must transport the payload between $2.2$ and $2.8$ meters per second to reach the target position without damaging the payload through high-frequency vibrations \cite{el2021high}. The vehicle will travel along various unknown surfaces, including dirt, gravel, and grass.  We demonstrate that Algorithm~\ref{alg: one_time} can be leveraged to solve Problems~\ref{prb: reach_avoid_problem} and~\ref{prb: partial_trajectory_tracking} under Conditions~\ref{req: unknown_dynamics} and~\ref{req: single_system_run}. 

\subsection{Control Design Requirements}

\begin{figure}[h]
    \centering
    \includegraphics[width=0.8\linewidth]{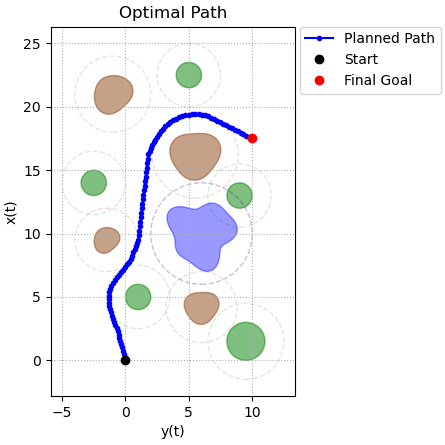}
    \caption{An optimal path determined by solving \eqref{eq: optimal_path_problem}. Path obstructing obstacles are shown in green, brown, and blue   with the tolerance radius $\B^2(c_k, r_k + \Delta)$ for all $k$ obstacles.}
    \label{fig: optimal_path}
\end{figure} 

We ultimately aim to navigate to a goal state while avoiding an unsafe set. Consider the states defined in \eqref{eq:kinematic_bicycle_model} and   notation in Problem \ref{prb: reach_avoid_problem}; we define the unsafe set as 
$\bb = \left\{(x,y,v): v\notin [2.2, 2.8]\;\text{and}\;(x,y)\in \mathcal{S}_{\text{blue}}\cup\mathcal{S}_{\text{brown}}\cup\mathcal{S}_{\text{green}}\right\}$ with the corresponding colored regions illustrated in Fig.~\ref{fig: optimal_path}. 
With $\bb$ defined, we formally present our performance objective, namely safely navigating to some goal state $\ox_f$.

\begin{obj}\label{obj: performance_objective}
    Synthesize a trajectory $\phi_u(\cdot,x_0)$ such that $\ox(0) \to \ox_f$ while $x(t) \not\in \bb$ for all $t$.
\end{obj}

We begin by determining an optimal path using sequential least squares programming (SLSQP) \cite{ma2022improved}. We solve the following optimization problem:

\begin{small}
    \begin{equation}\label{eq: optimal_path_problem}
\begin{gathered}
    \min_W J(W) = \sum_{i=1}^{N-1}\|(x_{i+1},y_{i+1})-(x_i,y_i)\|^2\\
    \mathrm{s.t.}\quad \|p_i^{\mathrm{front}} - c_k\| \geq r_k + \Delta\quad \forall\,i,k
\end{gathered}
\end{equation}
\end{small}

\noindent where $W = \begin{bmatrix}(x_1, y_1),\hdots,(x_N,y_N)\end{bmatrix} \in \R^{N\times 2}$, $p_i^{\mathrm{front}} = (x_i,y_i) + \frac{L((x_i, y_i) - (x_{i-1},y_{i-1}))}{2\|(x_i, y_i) - (x_{i-1},y_{i-1})\|}$, $c_k$ and $r_k$ are the center and radius of the $k$-th obstacle, and $\Delta$ is the tolerance around each obstacle to avoid collisions. The optimal path we will follow and its surrounding environment are shown in Fig.~\ref{fig: optimal_path}. We now discuss how we can synthesize control action designed to follow the desired path for the vehicle presented in Section \ref{Sec: vehicle_dynamics}.

\subsection{Controller Synthesis for Vehicle Dynamics}

\begin{figure}[htbp]
    \centering
    \begin{subfigure}[t]{0.24\textwidth}
        \centering
        \includegraphics[width=\linewidth]{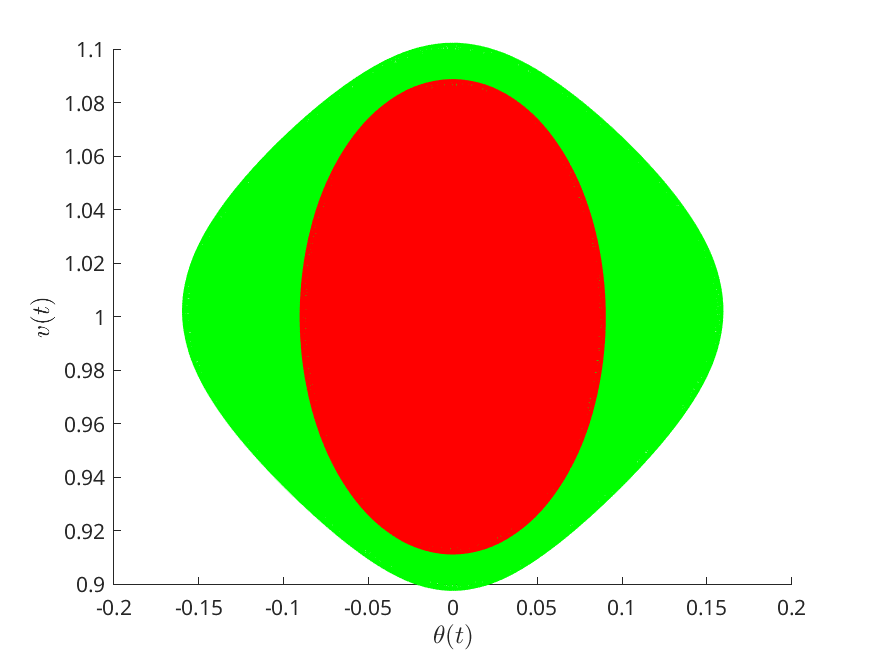}
        \label{fig:car_actuated_grs_underapproximations}
    \end{subfigure}%
    \hfill
    \begin{subfigure}[t]{0.24\textwidth}
        \centering
        \includegraphics[width=\linewidth]{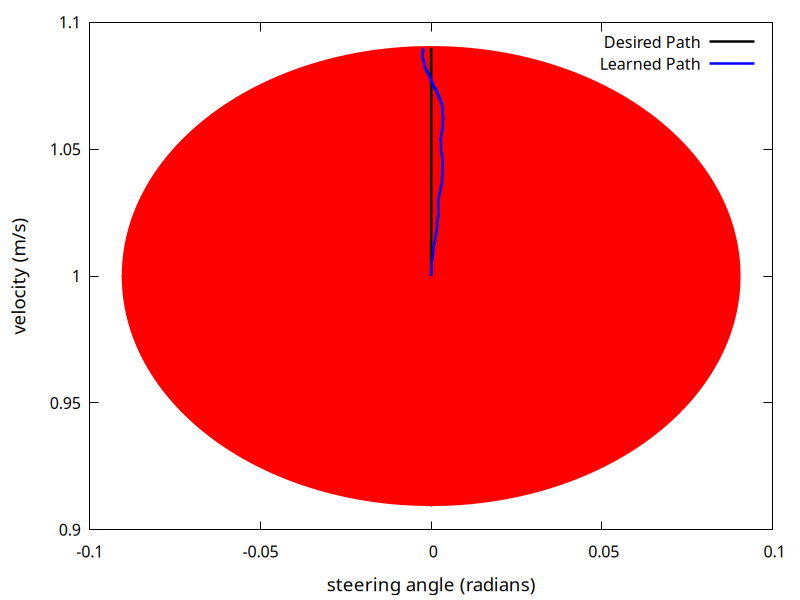}
        \label{fig:path_following_velocity_example}
    \end{subfigure}
    
    \caption{On the left, the true reachable set (green) of actuated states $\theta,v$ of system \eqref{eq:kinematic_bicycle_model} calculated knowing the dynamics, and an underapproximation (red) calculated using solely the knowledge from Assumption \ref{ass: growth_rate_bounds} for system \eqref{eq:simple_bicycle_model}. On the right, we apply Algorithm \ref{alg: one_time} to learn a path (blue) designed to follow a reachable path (black). The learned path (blue) is calculated without knowing the dynamics, but instead using $\tx_0 = \begin{bmatrix}0 & 1\end{bmatrix}^T$, $L_f^\nn = 0$, $L_G^\nn = 3$, $\tf(\tx_0) = 0$, $\tG(\tx_0) = I$.}
    \label{fig: grs_actuated_underapproximation}
\end{figure}

As discussed in Section~\ref{Sec: vehicle_dynamics}, the vehicle dynamics in \eqref{eq:kinematic_bicycle_model} are not control-affine. However, following the philosophy of GRS underestimation via proxy systems and employing myopic-type control for learning and control as in Section \ref{sec: underapproximated_guarantees} and \ref{Sec: algorithm}, our algorithm operates iteratively over short time horizons. For small angles, $\sin{\theta} \approx \theta$, $\cos{\theta} \approx 1$. Additionally, $|u_2| \leq |\tan({u_2})|$ for $|u_2| \leq 1$, and the reachable set of \eqref{eq:simple_bicycle_model} 
\begin{equation}\label{eq:simple_bicycle_model}
    \dx  = v,\;
        \dy  = v\theta,\;
        \dot{\theta}  = \frac{v}{l_r}u_2,\;
        \dv  = u_1
\end{equation}
is contained in the reachable set of \eqref{eq:kinematic_bicycle_model} (as shown in Fig. \ref{fig: grs_actuated_underapproximation}. Note that \eqref{eq:simple_bicycle_model} satisfies the assumed structure. This construction is valid for RPC purposes, as it provides an underapproximation of the GRS of \eqref{eq: nonlinaer_control_affine_system}, ensuring navigation toward a provably reachable state. Operating within this underapproximated GRS prevents the system from entering unproved regions of the state space. We can thus apply the procedures from \cite{ornik2019control, el2023online} to extract knowledge from Assumption~\ref{ass: growth_rate_bounds} for the local system~\eqref{eq:simple_bicycle_model}. The above setting ensures that our guarantees remain valid and enables controller synthesis, as illustrated in Fig.~\ref{fig: grs_actuated_underapproximation}. 

We ultimately need to follow this path while $2.2\,m/s < v < 2.8\,m/s$ to satisfy Objective \ref{obj: performance_objective}. We now detail our proposed control strategy to accomplish this task. 

\subsection{RPC Algorithm and Control Strategy}
We begin by presenting the RPC algorithm, a high-level path-following algorithm that requires iteratively applying the low-level reachability control Algorithm~\ref{alg: one_time}.
We use the notation $\hat{n} \in \nn$, where $\hat{n}$ indexes the number of times Algorithm~\ref{alg: one_time} has been applied, and $\nn$ is the integer set of all iteration indices, starting at index $0$. 

\begin{algorithm}
    \caption{Reachable Predictive Control}\label{alg: alg_rpc}
    \begin{algorithmic}[1] 
        \Require $d$, $m$,  $\tx_{0, 0}:=\tx_0$, $T$, $\tx_f^{0}\in\partial \rr_{\tx}(T,\tx_0)$, $\ox_f$, $\theta_0~=~0$, and $u_{0, 0}= (1-\epsilon) \frac{\tG^\dagger(\tx_0) (\tx_f -\tx_0)}{\|\tG^\dagger(\tx_0)\||\tx_f -x_0|}$, $\overline{r}~=~ \sup_{\ox \in \rr_{\ox}(T,x_0)}\|\ox_0 - \ox\|$, $\hat{n}=0$. 
        \Require  $\dt$, $\epsilon$, $k$  based on the condition \eqref{eq: convergence_condition}, where $r~=~\sup_u|\tx_p(k(m+1)\dt)|$
        \While{$\|\ox - \ox_f\| > \overline{r}$}
            \State Run Algorithm \ref{alg: one_time} with $\tx_f^{\hat{n}}$
                \State Update $x_0$ to current state 
                \State Calculate $\rr_{\tx}(T,\tx_0)$ using \cite[Theorem 3]{shafa2022reachability}
               
                \State Calculate $\rr_{\ox}(T,x_0)$ in using Theorem \ref{thm: unactuated_grs_approximation}
                \State Choose $\tx_f^{\hat{n} + 1} \in \rr_{\tx}(T,\tx_0)$
                \State Reset $r~=~\sup_{u}|\tx_p(k(m+1)\dt)|$ for Algorithm \ref{alg: one_time}
                 \State Reinitialize   $\theta_0 = 0$ for Algorithm \ref{alg: one_time}
                \State $\overline{r}~=~\sup_{\ox \in \rr_{\ox}(T,x_0)}\|\ox_0 - \ox\|$
                
        \State $\hat{n}:=\hat{n}+1$
        
        \EndWhile
    \end{algorithmic}
\end{algorithm}

To correctly apply Algorithm~\ref{alg: alg_rpc}, we must reinitialize the initial conditions at the end of each iteration to the system’s current state,  update the estimates of $\rr_{\tx}(T,\tx_0)$ and $\rr_{\ox}(T,x_0)$, and to choose $\tx_f^{\hat{n} + 1}$, such that Algorithm~\ref{alg: one_time} can be used to synthesize a controller that reaches $\tx_f^{\hat{n} + 1}$. The selection of $\tx_f^{\hat{n} + 1}$ is guided by the updated information of $\tx_f^{\hat{n} + 1}$ and its relative position w.r.t. the desired path, which we look over multiple time horizons for high-level planning. Physically, this corresponds to tuning the velocity and steering angles so that the controlled state trajectory   converges to the desired path. The following remark provides conditions under which the calculated reachable sets are simply translated, that is, situations where the calculated sets are affine transforms centered at different initial conditions.


\begin{rem}\label{rem: grs_multiple_horizons}
    Consider notation consistent with Theorem \ref{thm: full_grs_underapproximation} and notation where given $\overline{a} \in \R^m$ and a set $\mathcal{A} \in \R^m$, $\mathcal{A} + \overline{a} = \{a + \overline{a}~|~a \in \mathcal{A}\}$. Let $L_f^{\hat{n}}$ and $L_G^{\hat{n}}$ be the Lipschitz constants around a neighborhood of $x(0)$ and similarly $L_f^{\hat{n} + 1}$ and $L_G^{\hat{n} + 1}$ be the same around a different neighborhood centered at $y(0)$. For $x,\,y \in \R^n$ such that $x(0) \neq y(0)$, if $\|\tG^\dagger(\tx_0)\|^{-1} = \|\tG^\dagger(\ty_0)\|^{-1}$, $\|F^\dagger(x_0)\|^{-1} = \|F^\dagger(y_0)\|^{-1}$, $L_f^{\hat{n}} = L_f^{\hat{n} + 1}$ and $L_G^{\hat{n}} = L_G^{\hat{n} + 1}$, then $\rr_{\tx}(T,\tx_0) + (\ty_0 - \tx_0) = \rr_{\ty}(T,\ty_0)$. Additionally, if $\|\tx_0\| = \|\ty_0\|$, then $\overline{\rr}_{\ox}(T,x_0) + (y_0 - x_0) = \overline{\rr}_{\overline{y}}(T,y_0)$. \hfill $\square$
\end{rem}

\begin{figure}[htbp]
    \centering
    \begin{subfigure}[t]{0.24\textwidth}
        \centering
        \includegraphics[width=\linewidth]{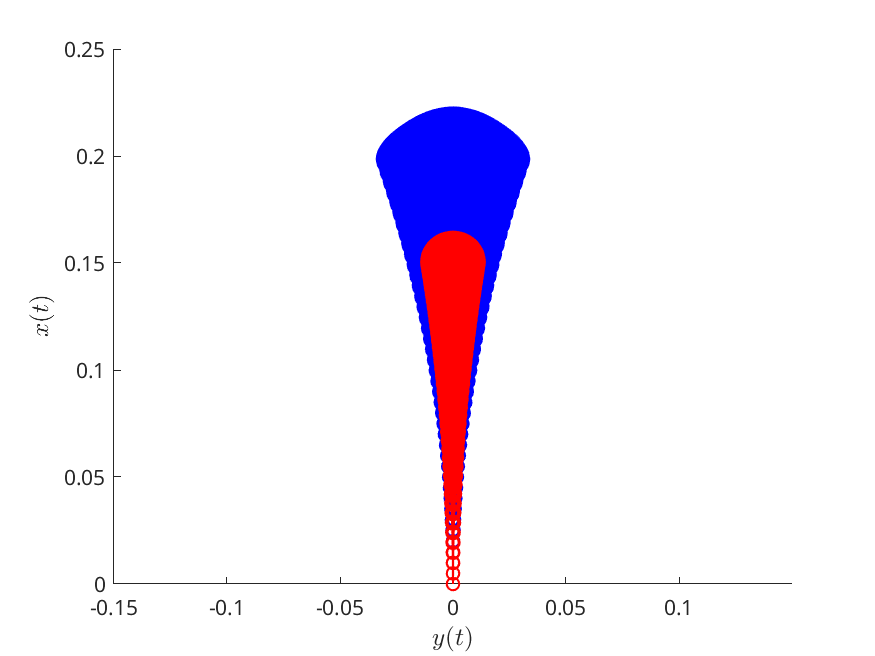}
        \label{fig:car_actuated_grs_underapproximations}
    \end{subfigure}%
    \hfill
    \begin{subfigure}[t]{0.24\textwidth}
        \centering
        \includegraphics[width=\linewidth]{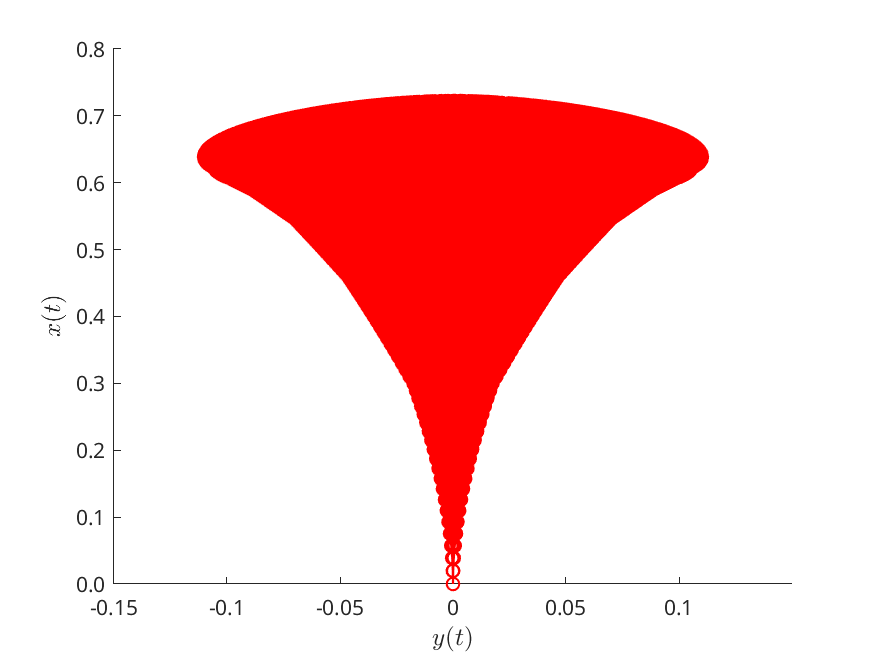}
        \label{fig:path_following_velocity_example}
    \end{subfigure}
    
    \caption{On the left, we have the reachable set (blue) of states $\ox = \begin{bmatrix}x & y\end{bmatrix}^T$ of system \eqref{eq:simple_bicycle_model}, calculated knowing the dynamics, and its \textit{guaranteed} underapproximation (red). On the right, we calculate $\rr_{\ox}^{\tn}(T,x_0)$. The sets in red are calculated without knowledge of the dynamics, using solely the knowledge of $x(0) = \begin{bmatrix}0 & 0 & 0 & 2.0\end{bmatrix}^T$ for $T = 0.1$ seconds, $\tf(\tx_0) = 0$, $\tG(\tx_0) = \operatorname{diag}(1,2)$, $L_f^{\hat{n}} = 0$, and $L_G^{\hat{n}} = 3$ for $\hat{n} \in \{1,\hdots,5\}$.}
    \label{fig: grs_position_underapproximation}
\end{figure}

Intuitively, for short time horizons, the values in Remark \ref{rem: grs_multiple_horizons} may not change significantly, making it feasible to approximate a larger set of states for path planning. We emphasize that controller synthesis is only computed using the guaranteed set of states $\rr_{\tx}(T,\tx_0)$; thus, the paths we follow remaint \textit{guaranteed} reachable paths. However, for high-level planning, we can closely approximate a set of states that are reachable to help inform which guaranteed path to follow. Let $\rr_{\ox}^{\tn}(T,x_0)$ be the reachable set $\rr_{\ox}(T,x_0)$ calculated under the assumption that Remark~\ref{rem: grs_multiple_horizons} holds for $\tn$ time horizons of length $T$ at $x(0) = x_0$. Fig. \ref{fig: grs_position_underapproximation} illustrates the reachable sets calculated at $x_0 = \begin{bmatrix}0 & 0 & 0 & 2.0\end{bmatrix}^T$ for $T = 0.1$ and $\tn = 5$. We use $\rr_{\ox}^{5}(0.1,x_0)$ to guide the selection of $\tx_f^{\hat{n} + 1}$ in Algorithm \ref{alg: alg_rpc} at each iteration until Objective \ref{obj: performance_objective} is completed. This strategy is formally presented in the table below.

\begin{table}[h!]
\centering
\begin{tabular}{|c|c|}
        \hline
        \textbf{Control Action} & \textbf{Condition} \\
        \hline
        Increase $v$ & $\rr_{\ox}^5(0.1,x_0)\,\cap$ desired path $\neq \varnothing$, $v \leq 2.3 (m/s)$ \\
        \hline
        Zero Input & $\rr_{\ox}^5(0.1,x_0)\,\cap$ desired path $\neq \varnothing$, $v \geq 2.7 (m/s)$\\
        \hline
        Increase $\theta$ & red $\cap$ green $= \varnothing$, car to left of desired path \\
        \hline
        Decrease $\theta$ &  red $\cap$ green $= \varnothing$, car to right of desired path \\
        \hline
    \end{tabular}
    \caption{Strategy for choosing $\tx_f^{\hat{n}+1}$ in Algorithm \ref{alg: alg_rpc}.}
    \label{table: control_strategy}
\end{table}

\begin{figure}[t]
  \centering
  \includegraphics[width=0.48\linewidth]{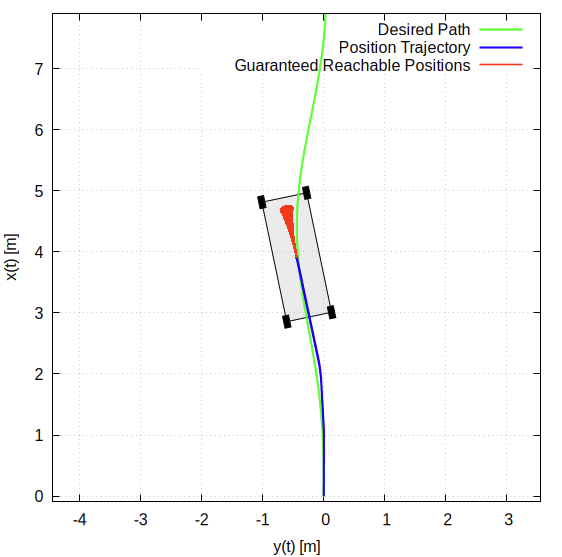}\hfill
  \includegraphics[width=0.48\linewidth]{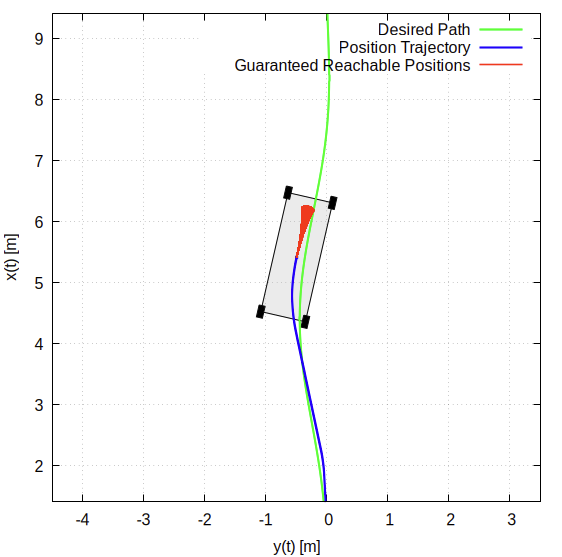}
  \caption{Autonomous path following of the desired path (green) by informed decision making using calculated guaranteed reachable positions (red) over a total of $0.5$ seconds.}
  \label{fig: path_following}
\end{figure}

Fig. \ref{fig: path_following} illustrates the  strategy shown in Table \ref{table: control_strategy}. Intuitively, $\tn$ can be viewed as a tuning knob that adjusts how closely the vehicle should follow the position path; the smaller $\tn$, the more often the control effort would focus $\theta$ to remain on the desired position path. We next detail how we can tune the parameters of Algorithms \ref{alg: one_time} and \ref{alg: alg_rpc} to provide a solution to design objective introduced in Objective \ref{obj: performance_objective}.

\subsection{Real-Time Performance}

\begin{figure}[htbp]
    \centering
    
    \begin{subfigure}[t]{0.26\textwidth} 
        \vspace{0pt}
        \includegraphics[width=\textwidth]{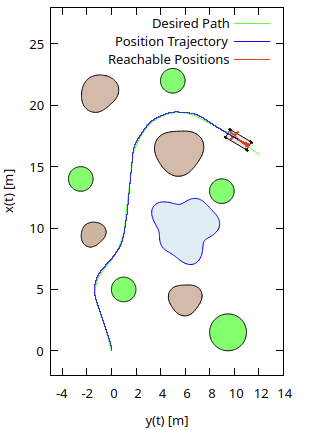}
        \label{fig:big}
    \end{subfigure}%
    \hfill
    \begin{subfigure}[t]{0.22\textwidth}
        \vspace{0pt}\includegraphics[width=\textwidth]{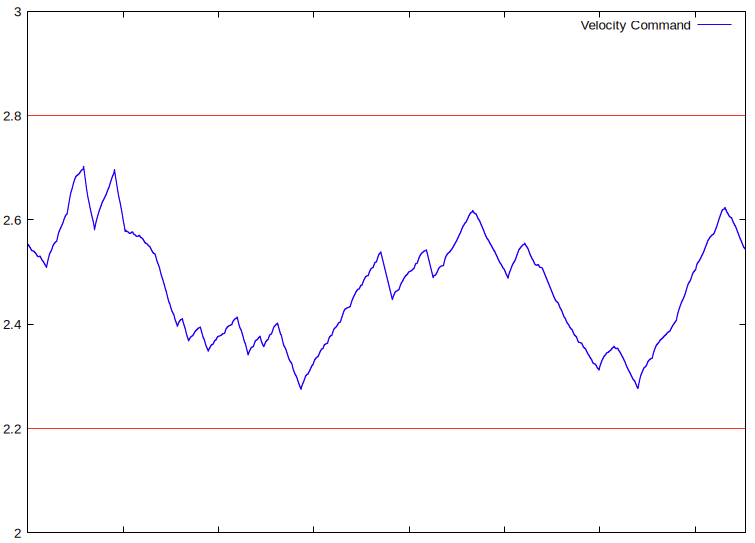}
        \label{fig:top}
        
        \vspace{0.2em} 
        
        \includegraphics[width=\textwidth]{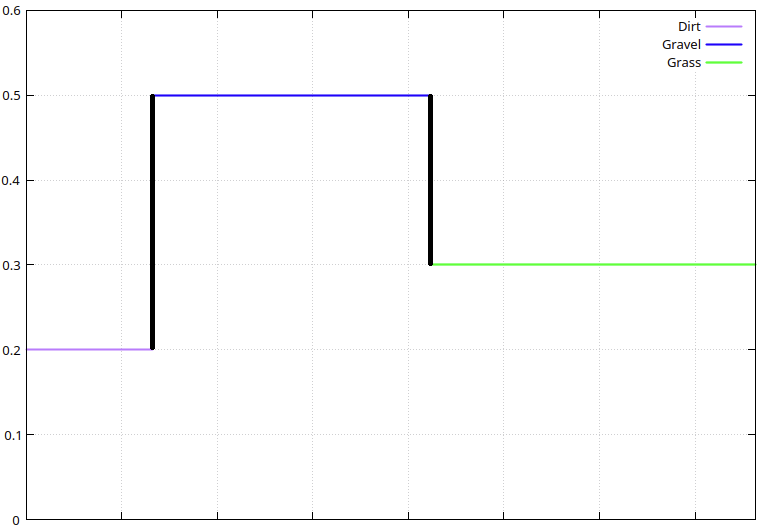}
        \label{fig:bottom}
    \end{subfigure}

    \caption{Unmanned vehicle position trajectory (blue) closely following an optimal desired path (green) guided by its reachable positions (red) without any obstacle collisions. The top-right graph shows the velocity over time such that $v(t) \not\in \bb$ for all $t$ while the bottom-right graph shows the varying unknown $r_c$ disturbance magnitudes as the vehicle traversed the unknown terrain.}
    \label{fig: RPC_Performance}
\end{figure}


For RPC to operate in real time, it is necessary to compute $\rr_{\tx}(T,\tx_0)$, $\rr_{\ox}(T,x_0)$, $\rr_{\ox}^{\tn}(T,x_0)$ rapidly. This is feasible due to the simple structure of the proxy systems \eqref{E: proxy} and \eqref{eq: ode_unactuated_grs_underapproximation}. Existing results \cite[Section~3]{meng2024online} show that, to characterize the boundary of the reachable sets of these systems, it suffices to consider inputs on the boundary of $\uu$. Consequently, we can efficiently compute the reachable sets by employing Monte Carlo simulations that solve ODEs with randomly sampled time-varying inputs restricted to the boundary of $\uu$. This significantly reduces computation time of $\rr_{\tx}(T,\tx_0)$, $\rr_{\ox}(T,x_0)$, $\rr_{\ox}^{\tn}(T,x_0)$. Additionally, as shown in Algorithm~\ref{alg: one_time}, the surrogate suboptimal control can be obtained via linear programming, achieving computation times comparable to those for the GRS computation. For the purposes of this application, we set $T = 0.1$ and $\tn = 5$ and for each iteration $\hat{n}$, we had a total computation time of $\approx 0.081$ seconds to calculate all results of interest. For additional background and analysis about how Algorithm \ref{alg: one_time} can be calculated in real time, we refer the reader to \cite{meng2024online}.


Throughout the simulation, we used $L_f^{\hat{n}} = 0$ and $L_G^{\hat{n}}~=~3$ for all $\hat{n}$, with $\tf(\tx_0)$ and $\tG(\tx_0)$ updated at each $\hat{n}$ depending on the reinitialized value of $x_0$ at the start of each iteration of Algorithm \ref{alg: one_time}. We simulated results using the dynamics from \eqref{eq:kinematic_bicycle_model} with $r_c = 0.2,\,0.5,\,0.3$, which pertains to the rolling resistance coefficients of dirt, gravel, and grass respectively. The reachable positions for the high-level planning are calculated using Theorem \ref{thm: unactuated_grs_approximation}. Notice that, adopting the notation of Theorem \ref{thm: generalized_gvs_ball_underapproximation}, this theorem requires that $b(t) \leq \|a\|$ where $\|a\| = \|\tx_0\|$ and $b(t) = \|y - \tx_0\|$ for any $y \in \partial\B^m(\tx_0; \|y - \tx_0\|) \subseteq \rr_{\tx}(T,\tx_0)$. These conditions are always met because $2.2 \leq v_0 \leq 2.8$ and because Algorithm \ref{alg: one_time} uses short time horizons, so $b(t)$ is always sufficiently small. In practice, if $\|\tx_0\|$ were small, the system would be near rest and we argue there would likely be no need for the implementation of safety-critical path following for a system at rest. In other words, in practice, one can expect the conditions for Theorem~\ref{thm: unactuated_grs_approximation} will be satisfied given $\|\tx_0\|$ is not small and that Algorithm~\ref{alg: one_time} works in short time horizons.

Fig. \ref{fig: RPC_Performance} illustrates the results for a randomly generated set of obstacles with varying terrain. The analysis from Section \ref{Sec: algorithm} and \cite{meng2024online} showed that, if the inequality \eqref{eq: convergence_condition} was satisfied, then there exist parameters $\dt$, $\epsilon$, and $k$ such that there is a minimizing solution to $u_{n+1,0}=(1-\epsilon)\operatorname{argmin}_{\lambda}\langle 2(\tx(\tau_{n+1})-z_{n+1}), \sum_{j=0}^m\lambda_j(\tx_{n,j+1}-\tx_{n,j}) \rangle,$
where $u_{n+1,0}$ drives the system towards the desired state $\tx_f^{\hat{n}}$ at every $n$. By setting $\dt = 0.015$, $\epsilon = 0.1$, and $k = 5$, Algorithm \ref{alg: one_time} was able to follow a desired trajectory to $\tx_f^{\hat{n}}$ at each iteration. Implementing the control strategy from Table \ref{table: control_strategy}, we see that RPC is capable of closely following the desired optimal path without ever intersecting an unsafe set by iteratively applying Algorithm \ref{alg: one_time}. Eventually, RPC completed Objective \ref{obj: performance_objective} without any position obstacle collisions while the vehicle velocity remained within the desired bound. 
\section{Conclusion}



In this paper, we have developed a Reachable Predictive Control (RPC) framework, illustrated with an example of an unknown vehicle model under reasonably mild assumptions—requiring only knowledge of the bounds on the Lipschitz constants. The proposed framework is capable of driving the system to follow a piecewise-linear trajectory on-the-run without prior knowledge of the system dynamics. In particular, we extend previous work from the control of fully actuated systems to the more challenging setting of underactuated systems. This includes a specific extension of GRS underapproximation techniques and a demonstration of how the control synthesis algorithm can be embedded into the RPC framework to enable an automated path-following procedure. The framework can generate control inputs automatically, using tunable parameters that allow the perturbations arising from learning and control to be made arbitrarily small. Moreover, it provides a principled approach for handling  abrupt changes in the system dynamics, enabling online learning and control with provable guarantees.

This work fills an important gap for scenarios where trajectory data is scarce, and traditional system identification methods or reinforcement learning cannot be applied effectively. The RPC framework complements and integrates well with MPC/CBF-based methods. Finally, for future engineering applications, it is of mathematical interest to sharpen the estimation bounds, thereby simplifying parameter tuning and improving practical usability.

\bibliographystyle{IEEEtran}
\bibliography{root}

\begin{thebibliography}{10}
\providecommand{\url}[1]{#1}
\csname url@samestyle\endcsname
\providecommand{\newblock}{\relax}
\providecommand{\bibinfo}[2]{#2}
\providecommand{\BIBentrySTDinterwordspacing}{\spaceskip=0pt\relax}
\providecommand{\BIBentryALTinterwordstretchfactor}{4}
\providecommand{\BIBentryALTinterwordspacing}{\spaceskip=\fontdimen2\font plus
\BIBentryALTinterwordstretchfactor\fontdimen3\font minus \fontdimen4\font\relax}
\providecommand{\BIBforeignlanguage}[2]{{%
\expandafter\ifx\csname l@#1\endcsname\relax
\typeout{** WARNING: IEEEtran.bst: No hyphenation pattern has been}%
\typeout{** loaded for the language `#1'. Using the pattern for}%
\typeout{** the default language instead.}%
\else
\language=\csname l@#1\endcsname
\fi
#2}}
\providecommand{\BIBdecl}{\relax}
\BIBdecl

\bibitem{lin2024one}
A.~Lin, S.~Peng, and S.~Bansal, ``One filter to deploy them all: Robust safety for quadrupedal navigation in unknown environments,'' \emph{arXiv preprint arXiv:2412.09989}, 2024.

\bibitem{ramdani2009computing}
N.~Ramdani and N.~S. Nedialkov, ``Computing reachable sets for uncertain nonlinear hybrid systems using interval constraint propagation techniques,'' \emph{IFAC Proceedings Volumes}, vol.~42, no.~17, pp. 156--161, 2009.

\bibitem{rungger2018accurate}
M.~Rungger and M.~Zamani, ``Accurate reachability analysis of uncertain nonlinear systems,'' in \emph{Proceedings of the 21st international conference on hybrid systems: Computation and control (part of CPS week)}, 2018, pp. 61--70.

\bibitem{knoedler2025safety}
L.~Knoedler, O.~So, J.~Yin, M.~Black, Z.~Serlin, P.~Tsiotras, J.~Alonso-Mora, and C.~Fan, ``Safety on the fly: Constructing robust safety filters via policy control barrier functions at runtime,'' \emph{IEEE Robotics and Automation Letters}, 2025.

\bibitem{rawlings2020model}
J.~B. Rawlings, D.~Q. Mayne, M.~Diehl \emph{et~al.}, \emph{Model predictive control: theory, computation, and design}.\hskip 1em plus 0.5em minus 0.4em\relax Nob Hill Publishing Madison, WI, 2020, vol.~2.

\bibitem{mayne2000constrained}
D.~Q. Mayne, J.~B. Rawlings, C.~V. Rao, and P.~O. Scokaert, ``Constrained model predictive control: Stability and optimality,'' \emph{Automatica}, vol.~36, no.~6, pp. 789--814, 2000.

\bibitem{mayne2005robust}
D.~Q. Mayne, M.~M. Seron, and S.~V. Rakovi{\'c}, ``Robust model predictive control of constrained linear systems with bounded disturbances,'' \emph{Automatica}, vol.~41, no.~2, pp. 219--224, 2005.

\bibitem{bemporad1999control}
A.~Bemporad and M.~Morari, ``Control of systems integrating logic, dynamics, and constraints,'' \emph{Automatica}, vol.~35, no.~3, pp. 407--427, 1999.

\bibitem{quan2021robust}
Y.~S. Quan, J.~S. Kim, and C.~C. Chung, ``Robust model predictive control with control barrier function for nonholonomic robots with obstacle avoidance,'' in \emph{2021 21st International Conference on Control, Automation and Systems (ICCAS)}.\hskip 1em plus 0.5em minus 0.4em\relax IEEE, 2021, pp. 1377--1382.

\bibitem{rosolia2018data}
U.~Rosolia, X.~Zhang, and F.~Borrelli, ``Data-driven predictive control for autonomous systems,'' \emph{Annual Review of Control, Robotics, and Autonomous Systems}, vol.~1, no.~1, pp. 259--286, 2018.

\bibitem{wang2023learning}
M.~Wang, X.~Lou, and B.~Cui, ``Learning-based robust model predictive control with data-driven koopman operators,'' \emph{International Journal of Machine Learning and Cybernetics}, vol.~14, no.~9, pp. 3295--3321, 2023.

\bibitem{dittmer2022data}
A.~Dittmer, B.~Sharan, and H.~Werner, ``Data-driven adaptive model predictive control for wind farms: A koopman-based online learning approach,'' in \emph{2022 IEEE 61st Conference on Decision and Control (CDC)}.\hskip 1em plus 0.5em minus 0.4em\relax IEEE, 2022, pp. 1999--2004.

\bibitem{chiu2022collision}
J.-R. Chiu, J.-P. Sleiman, M.~Mittal, F.~Farshidian, and M.~Hutter, ``A collision-free mpc for whole-body dynamic locomotion and manipulation,'' in \emph{2022 international conference on robotics and automation (ICRA)}.\hskip 1em plus 0.5em minus 0.4em\relax IEEE, 2022, pp. 4686--4693.

\bibitem{cheng2019end}
R.~Cheng, G.~Orosz, R.~M. Murray, and J.~W. Burdick, ``End-to-end safe reinforcement learning through barrier functions for safety-critical continuous control tasks,'' in \emph{Proceedings of the AAAI conference on artificial intelligence}, vol.~33, no.~01, 2019, pp. 3387--3395.

\bibitem{zhao2021model}
W.~Zhao, T.~He, and C.~Liu, ``Model-free safe control for zero-violation reinforcement learning,'' in \emph{5th Annual Conference on Robot Learning}, 2021.

\bibitem{thananjeyan2021recovery}
B.~Thananjeyan, A.~Balakrishna, S.~Nair, M.~Luo, K.~Srinivasan, M.~Hwang, J.~E. Gonzalez, J.~Ibarz, C.~Finn, and K.~Goldberg, ``Recovery rl: Safe reinforcement learning with learned recovery zones,'' \emph{IEEE Robotics and Automation Letters}, vol.~6, no.~3, pp. 4915--4922, 2021.

\bibitem{he2024agile}
T.~He, C.~Zhang, W.~Xiao, G.~He, C.~Liu, and G.~Shi, ``Agile but safe: Learning collision-free high-speed legged locomotion,'' \emph{arXiv preprint arXiv:2401.17583}, 2024.

\bibitem{bharadhwaj2020conservative}
H.~Bharadhwaj, A.~Kumar, N.~Rhinehart, S.~Levine, F.~Shkurti, and A.~Garg, ``Conservative safety critics for exploration,'' \emph{arXiv preprint arXiv:2010.14497}, 2020.

\bibitem{hsu2023sim}
K.-C. Hsu, A.~Z. Ren, D.~P. Nguyen, A.~Majumdar, and J.~F. Fisac, ``Sim-to-lab-to-real: Safe reinforcement learning with shielding and generalization guarantees,'' \emph{Artificial Intelligence}, vol. 314, p. 103811, 2023.

\bibitem{jenelten2024dtc}
F.~Jenelten, J.~He, F.~Farshidian, and M.~Hutter, ``Dtc: Deep tracking control,'' \emph{Science Robotics}, vol.~9, no.~86, p. eadh5401, 2024.

\bibitem{miki2022learning}
T.~Miki, J.~Lee, J.~Hwangbo, L.~Wellhausen, V.~Koltun, and M.~Hutter, ``Learning robust perceptive locomotion for quadrupedal robots in the wild,'' \emph{Science robotics}, vol.~7, no.~62, p. eabk2822, 2022.

\bibitem{agarwal2023legged}
A.~Agarwal, A.~Kumar, J.~Malik, and D.~Pathak, ``Legged locomotion in challenging terrains using egocentric vision,'' in \emph{Conference on robot learning}.\hskip 1em plus 0.5em minus 0.4em\relax PMLR, 2023, pp. 403--415.

\bibitem{shafa2022reachability}
T.~Shafa and M.~Ornik, ``Reachability of nonlinear systems with unknown dynamics,'' \emph{IEEE Transactions on Automatic Control}, vol.~68, no.~4, pp. 2407--2414, 2022.

\bibitem{shafa2022maximal}
------, ``Maximal ellipsoid method for guaranteed reachability of unknown fully actuated systems,'' in \emph{2022 IEEE 61st Conference on Decision and Control (CDC)}.\hskip 1em plus 0.5em minus 0.4em\relax IEEE, 2022, pp. 5002--5007.

\bibitem{Shafa23Reachability}
------, ``Reachability of nonlinear systems with unknown dynamics,'' \emph{IEEE Transactions on Automatic Control}, vol.~68, no.~4, pp. 2407--2414, 2023.

\bibitem{ornik2019control}
M.~Ornik, S.~Carr, A.~Israel, and U.~Topcu, ``Control-oriented learning on the fly,'' \emph{IEEE Transactions on Automatic Control}, vol.~65, no.~11, pp. 4800--4807, 2019.

\bibitem{meng2024online}
Y.~Meng, T.~Shafa, J.~Wei, and M.~Ornik, ``Online learning and control synthesis for reachable paths of unknown nonlinear systems,'' \emph{IEEE Transactions on Automatic Control (conditionally accepted)}, 2025.

\bibitem{liu2013survey}
Y.~Liu and H.~Yu, ``A survey of underactuated mechanical systems,'' \emph{IET Control Theory \& Applications}, vol.~7, no.~7, pp. 921--935, 2013.

\bibitem{spong2020robot}
M.~W. Spong, S.~Hutchinson, and M.~Vidyasagar, ``Robot modeling and control,'' \emph{John Wiley \&amp}, 2020.

\bibitem{el2023online}
H.~El-Kebir, A.~Pirosmanishvili, and M.~Ornik, ``Online guaranteed reachable set approximation for systems with changed dynamics and control authority,'' \emph{IEEE Transactions on Automatic Control}, vol.~69, no.~2, pp. 726--740, 2023.

\bibitem{polack2017kinematic}
P.~Polack, F.~Altch{\'e}, B.~d'Andr{\'e}a Novel, and A.~de~La~Fortelle, ``The kinematic bicycle model: A consistent model for planning feasible trajectories for autonomous vehicles?'' in \emph{2017 IEEE intelligent vehicles symposium (IV)}.\hskip 1em plus 0.5em minus 0.4em\relax IEEE, 2017, pp. 812--818.

\bibitem{yuan2014trajectory}
J.~Yuan, H.~Chen, F.~Sun, and Y.~Huang, ``Trajectory planning and tracking control for autonomous bicycle robot,'' \emph{Nonlinear Dynamics}, vol.~78, no.~1, pp. 421--431, 2014.

\bibitem{matute2019experimental}
J.~A. Matute, M.~Marcano, S.~Diaz, and J.~Perez, ``Experimental validation of a kinematic bicycle model predictive control with lateral acceleration consideration,'' \emph{IFAC-PapersOnLine}, vol.~52, no.~8, pp. 289--294, 2019.

\bibitem{el2021high}
H.~El-Kebir, T.~Shafa, A.~Purushottam, M.~Ornik, and A.~Soylemezoglu, ``High-frequency vibration reduction for unmanned ground vehicles on unstructured terrain,'' in \emph{International Conference on Modelling and Simulation for Autonomous Systems}.\hskip 1em plus 0.5em minus 0.4em\relax Springer, 2021, pp. 74--92.

\bibitem{ma2022improved}
Y.~Ma, N.~Zhang, and J.~Li, ``Improved sequential least squares programming--driven feasible path algorithm for process optimisation,'' in \emph{Computer aided chemical engineering}.\hskip 1em plus 0.5em minus 0.4em\relax Elsevier, 2022, vol.~51, pp. 1279--1284.

\bibitem{strang2016introduction}
G.~Strang, \emph{Introduction to Linear Algebra}.\hskip 1em plus 0.5em minus 0.4em\relax Wellesley-Cambridge Press, 2016.

\bibitem{stewart1998perturbation}
G.~W. Stewart, ``Perturbation theory for the singular value decomposition,'' Tech. Rep., 1998.

\end{thebibliography}

\appendices
\section{Error Estimation and Growth Rate Bounds}\label{sec: error}

We summarize the perturbation terms in \eqref{eq: convergence_condition}  generated from system learning under   control of the form \eqref{E: u_nj}. 

Denote \begin{small}
    $M_0:= \max\{\sup_{\tx\in\B}|\tf(\tx)|, \sup_{\tx\in\B}|\tG(\tx)|\}$, $\lmax:=\max\{L_f^\mathcal{N}, L_G^\mathcal{N}\}$, $C:=\|\tG(\tx_0)\|\|\tG^\dagger(\tx_0)\|$, $C_0:=M_0(m+1)$, $C_1:=M_0(m+1)^2$, $C_2:=2M_0\lmax(m+1)^2$,  and 
    $C_3:=M_0\lmax(m+1)^3$. 
 Let \(\tx_{n, j}=\tx(\tau_n+j\dt)\) and $\xk_{n+1}:=\tx_{n,m+1}$. Let $v_x(u):=\tf(\tx)+\tG(\tx)u$. Then, it is clear that $\xk_{n+1}=\tx_{n+1, 0}$.  We   have the following approximation precision \cite[Lemma 4]{ornik2019control}: 
\begin{enumerate}
    \item[(1)] $|\tx(t_1)-\tx(t_2)|\leq  C_0|t_2-t_1|$, $\forall t_1, t_2\in[\tau_n, \tau_{n+1}]$. In particular, $|\xk_{n+1}-\xk_n|\leq C_1\cdot\delta t$; 
    \item[(2)] $|\frac{(\tx_{n,j+1}-\tx_{n,j})}{\dt}- v_{\tx_{n,j+1}}(u_{n,j})|\leq (C_2/4)\cdot\dt$, $\forall j\in\I$. 
    \item[(3)] $|v_{x_{n,j+1}}(u_{n,j})-v_{\xk_{n+1}}(u_{n,j})|\leq C_3\cdot\dt$,  $\forall j\in\I$.
\end{enumerate}

The perturbation terms in \eqref{eq: convergence_condition} are summarized as follows: $\ee_r(\dt, \epsilon):=C_2\dt+\epsilon M_0$, $\ee_n(\dt, \epsilon):=2M_0C_1\dt+C_1\ee_r(\dt, \epsilon)$,  $\ee_\mu(\dt,\epsilon)= (3L_d(C_4/C_3)  \ee_\lambda(\dt, \epsilon)
            +L_dC_0\dt)/2$,  where
     $\ee_\lambda(\dt, \epsilon)=2((4m^{\frac{3}{2}}+\epsilon)/\epsilon) \cdot C_3\cdot \dt$.
    \end{small} 

\newline

We introduce growth rate bounds on $F(x(t))$ from \eqref{eq: underactuated_nonlinear_system_b} using the knowledge available from Assumption~\ref{ass: growth_rate_bounds}.

\begin{lem}\label{lem: growth_rate_bound}
    Let $\tx(t) \in \B^m(\tx_0;b(t))$ where $b(t) \in \R_+$ and $\|\tx(t)\| > 0$. Let $\alpha_{\mathrm{min}} = \|\tx_0\| - b(t)$ and $\alpha_{\mathrm{max}} = \|\tx_0\| + b(t)$. If $L_F(\|x-x_0\|) = \frac{L_f^\nn\alpha_{\mathrm{min}}\alpha_{\mathrm{max}} + \|\of(x_0)\|(1 + 2\alpha_{\mathrm{max}})}{\alpha_{\mathrm{min}}^3} + \mathcal{O}(\|x - x_0\|)$ where $\mathcal{O}(\|x - x_0\|) = \frac{L_f^\nn(\alpha_{\mathrm{min}} + 2\alpha_{\mathrm{max}})}{\alpha_{\mathrm{min}}^3}\|x - x_0\|.$ Then $\|F(x) - F(x_0)\| \leq L_F(\|x - x_0\|)\|x - x_0\|$.
\end{lem}
\begin{proof}
    Consider $F(x):= \frac{\of(x)\tx^T}{\|\tx\|^2}$. Then, for any $x$, $y$, $\|F(x) - F(y)\| \leq \left\|\frac{\of(x)\tx^T - \of(y)\ty^T}{\|\tx\|^2}\right\| + \|\of(y)\ty^T\|\left|\frac{1}{\|\tx\|^2} - \frac{1}{\|\ty\|^2}\right| \leq \frac{\|\of(x) - \of(y)\|\|\tx\| + \|\of(y)\|\|\tx - \ty\|}{\|\tx\|^2} + \|\of(y)\|\frac{\left|\|\tx\|^2 - \|\ty\|^2\right|}{\|\tx\|^2\|\ty\|}.$ We have $\|\of(x) - \of(y)\| \leq L_f^\nn\|x - y\|$, $\|\tx - \ty\| \leq \|x - y\|$, $\|\of(y)\| \leq \|f(x)\| + L_f\|x - y\|$, and $\|\tx\| \leq \|\tx_0\| + c(t)$, so $\|\of(x) - \of(y)\|\|\tx\| + \|\of(y)\|\|\tx - \ty\| \leq (L_f^\nn(\|\tx_0\| + c(t)) + \|\of(x)\|)\|x - y\| + L_f^\nn\|x - y\|^2$. Additionally, using the property that $\left|\|\tx\|^2 - \|\ty\|^2\right| \leq (\|\tx\| + \|\ty\|)\|\tx - \ty\|$, we get $\|\of(y)\|\|\ty\|\left|\|\tx\|^2 - \|\ty\|^2\right| \leq 2\|\of(x)\|(\|\tx_0\| + b(t))\|x - y\| + 2L_f^\nn(\|\tx_0\| + b(t))\|x-y\|^2$. Finally, substituting the right-hand side of the inequality $\|\tx\|,\|\ty\| \geq \|\tx_0\| - b(t)$ provides our result.
\end{proof}

The RPC algorithm calculates reachable sets and performs controller synthesis for short time horizons, so in practice, $\|x - x_0\|$ is not large. Additionally, as $t\to 0$, $\alpha_{\max}\to\alpha_{\mathrm{min}}$, so in practice, $L_f^\nn$ often serves as an acceptable Lipschitz bound for $F(x)$. For cases where $L_f^\nn$ is not a viable Lipschitz bound for $F(x(t))$, we could use the bound from Lemma \ref{lem: growth_rate_bound}.
 \section{Additional Analysis}\label{sec: additional analysis}

\noindent \textbf{Proof of Lemma \ref{lem: spherical_reachable_set}:}

\begin{proof}
Let $x(t)$ and $z(t)$ be trajectories of the original and approximated systems, respectively, where $x(0) = z(0) = x_0$ with both systems using the same control input $u(t) \in \uu$. Suppose that $T$ is such that $\rr_z(T, x_0)\cup\rr_x(T, x_0)\subseteq \B(x_0, R)$ for some $R>0$. Let $r(t) = \|x(t)-x(0)\|$. Then, by this assumption, we have $r(t)\in[0, R]$ for all $t\in[0, T]$.

Now we suppose $\|a\|\leq b-cR$.  Observing that $$\dot{x}(t) = \dot{z}(t) + w(t), \quad \mathrm{where} \quad w(t):= a + \|a\|u(t),$$ we define the velocity sets for two systems as $V_x(r) = a+ (b-cr)\B$ and $V_z(r)=(b-cr-\|a\|)\B$, respectively, where we have slightly abused notation and written $r:=|x-x_0|$ for the instantaneous value appearing on the right-hand side of the differential equations for $r(t)$ introduced above. We claim that $V_z(r)\subseteq V_x(r) \Leftrightarrow \|a\|\leq b-cR$ for each $r\in[0, R]$. 

Indeed, for $\Leftarrow$ side, taking any $v\in V_z(r)$, we have $\|v\|\leq b-cr-\|a\|$. Then, $\|v-a\|\leq \|v\|+\|a\|\leq b-cr$, and therefore $v\in a + (b-cr)\B = V_x(r)$. Conversely, for $\Rightarrow$ side,  
$\mathbf{0}\in V_z(r)$ implies $\mathbf{0}\in V_x(r)$ by inclusion. Then, there exists $u\in\mathcal{U}$ such that $a + (b-cr)u = \mathbf{0}$, which implies $\|a\|\leq b-cr$ for all $r\in[0, R]$. 

The conclusion of $\rr_z(T,x_0) \subseteq \rr_x(T,x_0)$ follows by standard comparison for differential inclusions under the same initial condition. Following similar steps as in the proof of \cite[Corollary 3]{shafa2022reachability}, we can also conclude that $\rr_z(T,x_0)$ is some ball in $\B\subset\R^m$. 
\end{proof}

\noindent\textbf{Proof of Theorem \ref{thm: generalized_gvs_ball_underapproximation}:}

\begin{proof}
Set $\cap_{(\hat{f},\hat{G}) \in \mathcal{D}_{con}}\hat{f}(x) - f(x_0) + \hat{G}(x)\mathcal{U} \text{~equals~} \mathcal{V}^\mathcal{G}_x - f(x_0)$. On the other hand, $\cap_{(\hat{f},\hat{G}) \in \mathcal{D}_{con}} \hat{f}(x) - f(x_0) + \hat{G}(x)\mathcal{U} = \cap_{(\Tilde{f},\hat{G})\in \tilde{\mathcal{D}}_{con}}\Tilde{f}(x) + \hat{G}(x)\mathcal{U}$, where $\tilde{\mathcal{D}}_{con}$ is defined the same as before, just with the assumption that $f(x_0) = 0$. Thus, we can assume without loss of generality that $f(x_0) = 0$.

Let $d \in \mathrm{Im}(G(x_0)) = \mathrm{Im}(R)$ such that $\|d\| = 1$. We will prove that if $|k| \leq c(t)(\|G^\dagger(x_0)\|^{-1} - L_f^\nn\|x\| - L_G^\nn\|x\|)$, then equation

\begin{equation} \label{Ch8_eqn:Thm1_ProblemFormulation}
    k \cdot d = \hat{f}(x) + \hat{G}(x)u,
\end{equation}

\noindent where $(\hat{f},\hat{G}) \in \mathcal{D}_{con}$, admits a solution $u \in \mathcal{U}=\mathbb{B}^{m}(\|a\|;b(t))$.

We subtract $\hat{f}(x)$ from both sides of \eqref{Ch8_eqn:Thm1_ProblemFormulation}. Since $\hat{f}(x) \in \mathrm{Im}(R)$ by Assumption 1 and $kd \in \mathrm{Im}(R)$ by definition, then $kd - \hat{f}(x) \in \mathrm{Im}(R)$. Also, $\mathrm{Im}(\hat{G}(x)) = \mathrm{Im}(R)$ by Lemma~1. Hence, there exists a vector $\bar{u} \in \mathbb{R}^m$ such that $kd - \hat{f}(x) = \hat{G}(x)\bar{u}$. Now, through the rank-nullity theorem \cite{strang2016introduction}, we can write $\bar{u} = u + u_2$ where $u \in \mathrm{Im}(\hat{G}(x)^T)$ and $u_2 \in \mathrm{Ker}(\hat{G}(x))$. Thus, $\hat{G}(x)\bar{u} = \hat{G}(x)(u + u_2) = \hat{G}(x)u$; hence, $kd - \hat{f}(x) = \hat{G}(x)u$.

We multiply both sides of $kd - \hat{f}(x) = \hat{G}(x)u$ on the left by $\hat{G}^\dagger(x)$, resulting in $\hat{G}^\dagger(x) (kd - \hat{f}(x)) = \hat{G}^\dagger(x) \hat{G}(x)u$. The term $\hat{G}^\dagger(x) \hat{G}(x)u$ results in the projection of $u$ onto the $\mathrm{Im}(\hat{G}(x)^T)$ \cite{strang2016introduction}. Given that $u \in \mathrm{Im}(\hat{G}(x)^T)$, by definition of a projection, $\hat{G}^\dagger(x) (kd - \hat{f}(x)) = \hat{G}^\dagger(x) \hat{G}(x)u = u$. Thus, if we prove that:

\begin{equation}
    \|\hat{G}^\dagger(x)(k \cdot d-\hat{f}(x))\| \leq c(t),
    \label{Ch8_eqn:thm1_1}
\end{equation}

\noindent we will have $\|u\| \leq c(t)$, i.e., $u \in \mathcal{U}$. Utilizing $\|d\| = 1$ along with the product and triangle inequalities for matrices, we arrive at \eqref{Ch8_eqn:thm1_4} and \eqref{Ch8_eqn:thm1_5}:

\begin{align}
    \|\hat{G}^\dagger(x)(k \cdot d-\hat{f}(x))\| \leq |k|\|\hat{G}^\dagger(x)~ d\| + \|\hat{G}^\dagger(x) \hat{f}(x)\| \label{Ch8_eqn:thm1_4}, \\
    \leq |k|\|\hat{G}^\dagger(x)\| + \|\hat{G}^\dagger(x)\| \|\hat{f}(x)\|. \label{Ch8_eqn:thm1_5}
\end{align}

From \eqref{Ch8_eqn:thm1_5} it follows that the set of all $k$ that satisfy $|k| \|\hat{G}(x)^{\dagger}\| + \|\hat{G}(x)^{\dagger}\| \|\hat{f}(x)\| \leq c(t)$ is a subset of all $k$ that satisfy \eqref{Ch8_eqn:thm1_1}. In other words, if:

\begin{equation}
    |k| \leq c(t)\left(\|\hat{G}^\dagger(x)\|^{-1} - \|\hat{f}(x)\|\right),
    \label{Ch8_eqn:thm1_2}
\end{equation}

\noindent then $k$ satisfies \eqref{Ch8_eqn:thm1_1}. We note $\|\hat{G}^\dagger(x)\| \neq 0$ from the definition of the Moore-Penrose pseudoinverse and because $G(x) \neq 0$ from Lemma~1. 

By Weyl's inequality for singular values \cite{stewart1998perturbation} and Assumption 2, we obtain the following inequalities: $\|\hat{G}^\dagger(x)\|^{-1} \geq \|G^\dagger(x_0)\|^{-1} - L_G^\nn\|x\|$ and $\|\hat{f}(x)\| \leq L_f^\nn\|x\|$. Thus, since we assumed that $k$ satisfies:

\begin{equation}
    |k| \leq c(t)\left(\|G^\dagger(x_0)\|^{-1} - L_f^\nn\|x\| - L_G^\nn\|x\|\right),
    \label{Ch8_eqn:thm1_6}
\end{equation}

\noindent it satisfies \eqref{Ch8_eqn:Thm1_ProblemFormulation}.
\end{proof}

\end{document}